\documentclass{article}[]
\usepackage{ijcai24}

\usepackage{times}
\usepackage{soul}
\usepackage{url}
\usepackage[hidelinks]{hyperref}
\usepackage[utf8]{inputenc}
\usepackage[small]{caption}
\usepackage{graphicx}
\usepackage{amsmath}
\usepackage{amsthm}
\usepackage{booktabs}
\usepackage{algorithm}
\usepackage{algorithmic}
\usepackage[switch]{lineno}

\newtheorem{theorem}{Theorem}
\newtheorem{lemma}{Lemma}
\newtheorem{claim}{Claim}
\newtheorem{corollary}{Corollary}

\newtheorem{definition}{Definition}

\usepackage{todonotes}
\usepackage{mathtools}
\usepackage{tikz}
\usetikzlibrary{positioning,chains,fit,shapes,calc}
\usepackage{tikz-qtree}
\usepackage{amsmath,amsthm,amssymb}
\usepackage{hyperref}
\usepackage{todonotes}
\usepackage{setspace}
\usepackage{subfiles}

\newcommand{\Cov}{\mathsf{Cov}}
\newcommand{\N}{\mathbb{N}}
\newcommand{\C}{\mathcal{C}}

\newcommand{\Non}{\mathsf{NCC}}
\newcommand{\Un}{\mathsf{UCC}}
\newcommand{\Par}{\mathsf{Par}}
\newcommand{\tO}{\tilde{O}}
\newcommand{\tOmega}{\tilde{\Omega}}

\newcommand{\AC}{\textrm{AC}_m}
\newcommand{\ACp}{\textrm{AC}_p}

\newcommand{\sat}{\mathsf{sat}}
\newcommand{\supp}{\mathsf{supp}}
\newcommand{\var}{\mathsf{var}}
\newcommand{\dom}{\mathsf{dom}}

\newcommand{\perm}{\mathsf{perm}}
\newcommand{\rep}{\mathsf{rep}}

\title{Structured d-DNNF Is Not Closed Under Negation}

\author{
   Harry Vinall-Smeeth
    \affiliations
    Technische Universit\"at Ilmenau
    \emails
    harry.vinall-smeeth@tu-ilmenau.de
}

\begin{document}

\maketitle

\begin{abstract}
Both structured d-DNNF, introduced in \cite{DBLP:conf/aaai/PipatsrisawatD08}, and SDD, introduced in \cite{DBLP:conf/ijcai/Darwiche11}, can be exponentially more succinct than OBDD. Moreover, SDD is essentially as tractable as OBDD. But this has left two important open  questions. Firstly, does OBDD support more tractable transformations than structured d-DNNF \cite{DBLP:conf/aaai/PipatsrisawatD08}? And secondly, is structured d-DNNF more succinct than SDD \cite{DBLP:conf/uai/BeameL15}? In this paper, we answer both questions in the affirmative. For the first question we show that, unlike OBDD, structured d-DNNF does not support polytime negation, disjunction, or existential quantification operations. As a corollary, we deduce that there are functions with an equivalent polynomial-sized structured d-DNNF but with no such representation as an SDD, thus answering the second question. We also lift this second result to \emph{arithmetic circuits} (AC) to show a succinctness gap between PSDD and the monotone AC analogue to structured d-DNNF.
\end{abstract}

\section{Introduction}

Knowledge compilation aims to provide useful representations of Boolean functions (propositional knowledge bases). What `useful' means is context dependent and has, broadly speaking, three aspects. The first is succinctness: how big is our representation? The second is transformations. For instance, given a representation for $f$ and a representation for $g$ can we form a representation for $f \wedge g$ in polynomial time? The third is queries: given our representations what can we (efficiently) determine about our function? For example, given a representation for $f$ can we determine $|f^{-1}(1)|$ in polynomial time? These aspects may be in tension with one another; to get a representation which supports more queries or transformations we may have to accept increased size. A key task in knowledge compilation is to map out the trade-offs of using different representations. 

In the landmark paper \cite{DBLP:journals/jair/DarwicheM02}, it is shown that many well-studied representation formats are subsets of Boolean circuits in Negation Normal Form (NNF). Consequently, over the past two decades, research on representations within the AI community has focused on  classes which arise from imposing syntactic restrictions on NNF. Two  influential restrictions are \emph{decomposability} and \emph{determinism}, an NNF that satisfies both properties is called a d-DNNF. Such circuits support a large range of polynomial time queries such as clausal entailment and model enumeration. 

Another, older, representation format is the Ordered Binary Decision Diagram (OBDD) first introduced in \cite{DBLP:journals/tc/Bryant86}. In fact, OBDD is a subset of d-DNNF \cite{DBLP:journals/jair/DarwicheM02}. While a Boolean function may have an equivalent d-DNNF that is exponentially smaller than any equivalent OBDD,  in many practical settings the latter is preferred. There are two crucial reasons for this. Firstly, OBDDs that use a common variable order are closed under Boolean operations; this is useful for instance in \emph{bottom-up} approaches to the compilation of Boolean formulas, see e.g.\! \cite{somenzi2009cudd}. Secondly, OBDDs are \emph{canonical} which greatly simplifies the task of finding an optimal compilation; one just needs to find an optimal variable order. 

A natural question is whether there are compilation languages lying between OBDD and d-DNNF that are more succinct than OBDD but which have nicer properties than d-DNNF? This paper will be concerned with two such languages: \emph{structured} d-DNNF \cite{DBLP:conf/aaai/PipatsrisawatD08} and Sequential Decision Diagram (SDD) \cite{DBLP:conf/ijcai/Darwiche11}. SDD has become a popular representation format since they, like OBDD, are canonical and closed under Boolean operations \cite{DBLP:conf/ijcai/Darwiche11,DBLP:conf/aaai/BroeckD15}. Moreover, they may be exponentially more succinct than OBDD \cite{DBLP:conf/aaai/Bova16}. Structured d-DNNF, on the other hand, contains SDD as a subset and supports a polynomial time conjoin operation. They may, however, be exponentially more verbose than d-DNNF. 

One may then wonder: is there any advantage to using structured d-DNNF over SDD? To be precise are their functions which have polynomially sized representations as structured d-DNNFs but do not have such SDD representations? This is a question which has been raised since at least 2015 \cite{DBLP:conf/uai/BeameL15} and has received substantial interest, see \cite{DBLP:conf/pods/BovaS17,DBLP:journals/mst/BolligF21}, but has remained open until now. In this paper, we answer this question in the affirmative.

\begin{theorem} \label{thm: sep}
For every $n \in \mathbb{N}$, there exists a function $f$ with an equivalent structured d-DNNF of size $n$ such that any SDD equivalent to $f$ has size $n^{\tOmega(\log(n))}$.
\end{theorem}

We prove this by showing that structured d-DNNF is \emph{not} closed under negation which has been an open question in its own right since \cite{DBLP:conf/aaai/PipatsrisawatD08}.

\begin{theorem} \label{thm: main}
For every $n \in \N$, there exists a Boolean function $f$ with an equivalent structured d-DNNF of size $n$ and such that any structured DNNF equivalent to $\neg f$ has size $n^{\tOmega(\log n)}$.
\end{theorem}

Thus, we simultaneously show that there is an advantage to using SDD over structured d-DNNF. We similarly show that structured d-DNNF is not closed under disjunction or existential quantification thus completing the `knowledge compilation map' for structured d-DNNF. 

\emph{Arithmetic circuits (AC)} also play an important role in AI, particularly in \emph{probabilistic reasoning}. Here one prominent circuit type is PSDD \cite{DBLP:conf/kr/KisaBCD14}. As the name suggests, these are the AC analogue of SDD. PSDDs have several nice properties making them ripe for applications. For example, they support a polynomial time multiplication operation (analogous to the polynomial time conjoin operation for SDDs), which is useful for instance in compiling probabilistic graphical models \cite{DBLP:conf/nips/ShenCD16}. de Colnet and Mengel observed that in many cases separations between representations of Boolean functions can be extended to monotone AC in a straightforward manner \cite[Proposition 2]{DBLP:conf/kr/ColnetM21}.\footnote{We should note here that although the paper states the proposition as an if and only if, in fact only one of the directions holds. Luckily, this is the direction used in the rest of the paper and which we need to translate our results to ACs.} We exploit this to show a succinctness gap between PSDD and the monotone AC analogue to structured d-DNNF

Our proof of Theorem~\ref{thm: main} exploits a connection between knowledge compilation and \emph{communication complexity} which has been widely deployed in recent years,  see e.g. \cite{DBLP:conf/uai/BeameL15,bova2016knowledge,DBLP:journals/mst/AmarilliCMS20}. We start from the same piece of communication complexity as \cite{goos2021lower}, where  an analogous result for unambiguous finite automata (UFA) is obtained. However, while the size of UFAs is related to the \emph{fixed partition} communication complexity model the size of structured d-DNNF is related to another model: the \emph{best partition} communication complexity. We therefore adapt an ingenious construction from \cite{knop2017ips}, which allows one to lift results from the fixed partition model to the best partition model.\footnote{We should note that there is an older construction from \cite{DBLP:journals/jcss/LamR92} which also allows one to lift communication complexity results to the best partition model. However, this construction requires that the functions involved are \emph{paddable}; as far as we can tell, this is not the case for the functions we use.}

The rest of the paper is structured as follows. In Section~\ref{sec: prelim} we define our main objects of study. In Section~\ref{sec: CC} we introduce the communication complexity we need. Following this, in Section~\ref{sec: main} we prove our main theorem. Finally, in Section~\ref{sec: AC} we show how our results extend to ACs. 

\section{Formulas, NNF, structured d-DNNF and SDD} \label{sec: prelim}

\subparagraph{Formulas and Boolean functions}
Recall that a propositional formula is a DNF if it is a disjunction of conjunctions. We call each disjunct a \emph{term}. A DNF $\psi$ is a $k$-DNF if each term contains at most $k$-literals and \emph{unambiguous} if every assignment $\alpha : \var(\psi) \to \{0,1\}$ satisfies at most one term of $\psi$. Let $\sat(\psi)$ denote the set of satisfying assignments for a propositional formula $\psi$. We identify each propositional formula $\psi$, with a Boolean function with domain $\{0,1\}^{\var(\psi)}$ in the standard way, i.e., the function evaluates to 1 on input $\underline{x}$ iff $\underline{x} \in \sat(\psi)$. For $f$ a Boolean function, we write $\sat(f) := f^{-1}(1)$. We will be interested in the following transformations.

\begin{definition} \label{def: trans}
Let $f, g : \{0,1\}^X \to \{0,1\}$ be Boolean functions and $x\in X$. Then we write:
\begin{enumerate}
\item (negation) $\neg f$ to denote the Boolean function with $\sat(\neg f) = f^{-1}(0)$;
\item (existential quantification) $\exists x f$ to denote the Boolean function with $\sat(\exists x f) = \pi_Y(\sat(f))$, where $\pi$ denotes projection and $Y:=\{0,1\}^{X \setminus \{x\}}$;
\item (disjunction) $f \vee g$ to denote the Boolean function with $\sat(f \vee g) = \sat(f) \cup \sat(g)$ and
\item (conjunction) $f \wedge g$ to denote the Boolean function with $\sat(f \wedge g) = \sat(f) \cap \sat(g)$.
\end{enumerate}
\end{definition} 

\subparagraph{Negation Normal Form}
\begin{definition} \label{def: NNF}
A Boolean circuit in \emph{Negation Normal Form (NNF)} is a vertex-labelled directed acyclic graph with a unique source such that every internal node is a fan-in two $\wedge$- or $\vee$-node and whose leaves are each labelled by 0, 1, a variable $x$ or a negated variable $\neg x$.
\end{definition}
We define the size of $\C$, an NNF to be the number of vertices in the underlying graph and denote this by $|\C|$. Note that by expanding out an NNF circuit $\C$ we get a unique propositional formula which we denote by $\langle \C \rangle$. Further, we write $\var(\C)$ for the set of variables occurring in $\C$. It will be convenient to associate a set of variables $\dom(\C)$ to $\C$ which contains $\var(\C)$ (together with possibly other variables). Unless otherwise stated, we assume that $\dom(\C) = \var(C)$. We write $f_C: \{0,1\}^{\dom(\C)} \to \{0,1\}$ to denote the Boolean function computed by $\C$ in the obvious way. We say that $\C$ is \emph{equivalent} to $f_{\C}$ and define $\sat(\C) := \sat(f_{\C})$. If for some $\mathsf{C} \subseteq$ NNF, a Boolean function $f$ is equivalent to some $\C \in \mathsf{C} $ of size $s$ then we say that $f$ \emph{admits} a $\mathsf{C}$ of size $s$.

\subparagraph{Decomposability, determinism and structuredness} For a node $g$ of $\C$, we write $\C(g)$ for the subcircuit rooted at $g$. If $g$ is not a leaf we write $g_{\ell}$ (resp. $g_r$) for its left (resp. right) child. An NNF, $\C$, is \emph{decomposable} if for every $\wedge$-node $g \in \C$, $\var(g_{\ell}) \cap \var(g_r) = \emptyset$ \cite{DBLP:journals/jacm/Darwiche01}. $\C$ is \emph{deterministic} if for every $\vee$-node $g \in \C$, $\sat(\C(g_{\ell})) \cap \sat(\C(g_{r})) = \emptyset$, where we set $\dom(\C(g_{\ell})) = \dom(\C(g_{r}))=\dom(\C)$ \cite{DBLP:journals/jancl/Darwiche01}. The set of decomposable NNF is denoted by DNNF and the set of deterministic DNNF by d-DNNF. 

We now only need one ingredient to get to structured d-DNNF; for this, we need the notion of a \emph{v-tree}.

\begin{definition} \label{def: vtree}
A \emph{v-tree} over variables $X$ is a full, rooted, binary tree whose leaves are in 1-1 correspondence with the elements of $X$. 
\end{definition}

For a non-leaf node $t$ of a v-tree $T$, we write $t_{\ell}$ for its left child and $t_{r}$ for its right child. A DNNF $\C$ \emph{respects} a v-tree $T$, if for every $\wedge$-node $g \in \C$, there is a node $t$ of $T$ such that $\var(g_{\ell}) \subseteq \var(t_{\ell})$ and $\var(g_{r}) \subseteq \var(t_{r})$; see Figure~\ref{fig: tree}.

\begin{definition} \label{def: structure}
A (d)-DNNF $\C$ is \emph{structured} if it respects some v-tree. We denote the set of structured (d)-DNNF by (d)-SDNNF.
\end{definition}

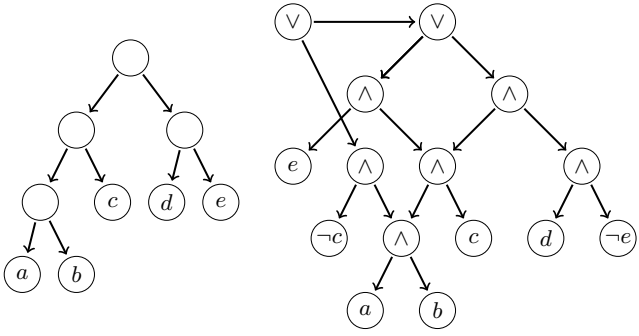
\begin{figure} 
\centering
\scalebox{0.96}{
\begin{tikzpicture}
  [wire/.style={thick, ->, shorten <= 0.3mm, shorten >= 0.5mm},
  gwire/.style={thick, ->, shorten <= 0.3mm, shorten >= 0.5mm, color=green},
   edge/.style={ -stealth, shorten <= 0.3mm, shorten >= 0.5mm},
   inputgate/.style={inner sep=1pt,minimum size=4mm},
   gate/.style={draw,circle,inner sep=1pt,minimum size=5mm},
   vertex/.style={draw,circle,fill=white,inner sep=2pt,minimum
   size=20mm}] 
   
    \begin{scope}[xshift=0.25cm,yshift=-1.5cm]
      \node[gate] (a) at (-1.5,0) {\small $a$};
      \node[gate] (b) at (-1.5+0.75,0) {\small$b$};
      \node[gate] (g1) at (-1.25,1) {};
      \node[gate] (g2) at (-1.25+1,1) {\small $c$};
      \node[gate] (g4) at (-0.75,2) {};
      \node[gate] (g5) at (0.75,2) {};
      \node[gate] (g6) at (-0,3) {};
      \node[gate] (d) at (0.5,1) {\small $d$};
      \node[gate] (e) at (1.25,1) {\small$e$};
      
      \draw[wire] (g1) -- (a);
      \draw[wire] (g1) -- (b);
      \draw[wire] (g4) -- (g1);
      \draw[wire] (g4) -- (g2);
      \draw[wire] (g6) -- (g4);
      \draw[wire] (g6) -- (g5);
      \draw[wire] (g5) -- (d);
      \draw[wire] (g5) -- (e);

    \end{scope}

    \begin{scope}[xshift=5.5cm,yshift=-2cm]
       \node[gate] (a) at (-2,0) {\small $a$};
      \node[gate] (b) at (-2+1,0) {\small$b$};
      \node[gate] (nc) at (-2.5,1) {\small$\neg c$};
      \node[gate] (w1) at (-1.5,1) {\small $\wedge$};
      \node[gate] (w2) at (-2,2){\small $\wedge$};
      \node[gate] (c) at (-1.5+1,1) {\small $c$};
      \node[gate] (w3) at (-1,2) {\small$\wedge$};
      \node[gate] (w4) at (-2,3) {\small$\wedge$};
      \node[gate] (e) at (-3,2) {\small $e$};
      \node[gate] (g5) at (1,2) {\small$\wedge$};
      \node[gate] (g6) at (-0,3) {\small$\wedge$};
      \node[gate] (d) at (0.5,1) {\small $d$};
      \node[gate] (ne) at (1.5,1) {\small$\neg e$};
      \node[gate] (v) at (-1,4) {\small$\vee$};
      %\node[gate] (v1) at (-3,3){\small$\vee$};
      \node[gate] (v2) at (-3,4){\small$\vee$};

      \draw[wire] (w1) -- (a);
      \draw[wire] (w1) -- (b);
      \draw[wire] (w2) -- (nc);
      \draw[wire] (w2) -- (w1);
      \draw[wire] (w3) -- (w1);
      \draw[wire] (w3) -- (c);
      \draw[wire] (w4) -- (w3);
      \draw[wire] (w4) -- (e);
      \draw[wire] (g6) -- (w3);
      \draw[wire] (g6) -- (g5);
      \draw[wire] (g5) -- (d);
      \draw[wire] (g5) -- (ne);
      \draw[wire] (v) -- (g6);
      \draw[wire] (v) -- (w4);
      \draw[wire] (v) -- (w4);
       \draw[wire] (v2) -- (v);
       \draw[wire] (v2) -- (w2);
  
\end{scope}      
\end{tikzpicture}
}
\caption{(left) A v-tree $T$. (right) A structured d-DNNF $\C$ that respects $T$. $\langle \C \rangle =  (a \wedge b \wedge \neg c ) \vee (a \wedge b \wedge c \wedge e) \vee (a \wedge b \wedge c \wedge d \wedge \neg e))$}.
\label{fig: tree}
\end{figure}

\subparagraph{SDDs} SDDs are a subset of d-SDNNF which arise from imposing a stricter form of determinism and structurdness called \emph{strong determinism}. The idea is that SDDs respect a certain type of decomposition which generalises the well-known Shannon Decomposition on which OBDDs are based. 

\begin{definition} \label{def: decomp}
Let $f: \{0,1\}^Z \to \{0,1\}$ be a Boolean function and $X,Y \subseteq Z$ be disjoint sets of variables. Then if 
\[
f = \bigvee_{i=1}^n p_i(X) \wedge s_i(Y)
\]  
 then $\{(p_1, s_1), \dots, (p_n, s_n)\}$ is an $X$-decomposition for $f$ if $\vee_{i=1}^n p_i \equiv 1$, $p_i \wedge p_j \equiv 0$ for all $i\neq j$ and $p_i \not \equiv 0$ for all $i$.
\end{definition}

We can now define SDDs. 

\begin{definition} \label{def: SDD}
 Let $T$ be a v-tree over variables $Z$ with root $t$. An SDD respecting $T$ is a DNNF $\C$ with one of the following forms:
 \begin{itemize}
 \item $\C$ consists of a single node labelled by 0, 1, $x$ or $\neg x$, where $x \in Z$.
 \item The source of $\C$ is a $\vee$ node $g$ such that:
 \begin{enumerate}
 \item $\langle \C \rangle = \bigvee_{i=1}^n p_i(X) \wedge s_i(Y)$ where $\{(p_1, s_1), \dots, (p_n, s_n)\}$ is an $X$ decomposition for $f_{\C}$,
 \item $X\subseteq \var(g_{\ell})$, $Y \subseteq \var(g_r)$ and
 \item if $h \in \C$ with $\langle \C(h) \rangle = p_i(X)$ (resp. $s_i(Y)$) for some $i$ then $\C(h)$ is an SDD that respects $t_{\ell}$ (resp. $t_r$).
\end{enumerate}  
 \end{itemize}
 An SDD is an SDD that respects some v-tree.
\end{definition}  

It follows from the definition that SDDs are deterministic and structured. One can further show that SDDs admit conjunction, disjunction and complementation in polynomial time \cite{DBLP:conf/ijcai/Darwiche11}. These are the main facts we need; we include the full definition for context and because it is needed for the connection to arithmetic circuits,\footnote{The definition is somewhat cumbersome as we restrict ourselves to fan-in 2 nodes. We do this to make the overall presentation cleaner. Note, that circuits with unbounded fan-in conjunction and disjunction can be rewritten as fan-in 2 circuits with only a quadratic size blow-up.} see \cite{DBLP:conf/ijcai/Darwiche11,DBLP:journals/mst/BolligF21} for a more thorough introduction to SDDs.

\subparagraph{Succinctness} 
Since we want to compare the succinctness of different representations we need the following notion.

\begin{definition} \cite{DBLP:conf/ijcai/GogicKPS95}
Let $\mathsf{C}_1$ and $\mathsf{C}_2$ be subsets of NNF. We say that $
\mathsf{C}_1$ is \emph{at least as succinct} as $\mathsf{C}_2$ if there is 
a polynomial $p$, such that for every $\C \in \mathsf{C}_2$ there is an 
equivalent $\C' \in \mathsf{C}_1$ with $|\C'| \le p(|\C|)$. We write $
\mathsf{C}_1 \le \mathsf{C}_2$. We say that $\mathsf{C}_1$ is more succinct 
than $\mathsf{C}_2$, denoted $\mathsf{C}_1 < \mathsf{C}_2$, when $
\mathsf{C}_1 \le \mathsf{C}_2$ and $\mathsf{C}_2 \not \le \mathsf{C}_1$.
\end{definition} 

\section{Knowledge Compilation and Communication Complexity} \label{sec: CC}

Our proof will use machinery from \emph{communication complexity}, see \cite{KushilevitzNisan} for an introduction. Communication complexity is concerned with variants of the following scenario. We have two players, Alice and Bob, who would like to determine the value of a two-party function $f : \{0,1\}^n \times \{0,1\}^m$ on an input $(\underline{x},\underline{y})$. The twist is that Alice only has access to $\underline{x}$ and Bob only has access to $\underline{y}$. They aim to compute $f(\underline{x}, \underline{y})$ while communicating as few bits as possible. We now formally introduce all the notions from communication complexity we will need in Section~\ref{sec: main}.

Consider a function $f \colon Z \to \{0,1\}$ and a partition $\Pi = (X,Y)$ of $Z$. We will only consider \emph{balanced} partitions. For us this means that $|Z|/3 \le \min\{|X|, |Y|\}$. Then
a set $A \times B \subseteq X \times Y$ (with $A \subseteq X$ and $B \subseteq Y$) is called a \emph{rectangle} with respect to $\Pi$; we will also call these $\Pi$-rectangles.
We say that $\Pi$-rectangles $R_1, \ldots, R_k$ \emph{cover} a set $S \subseteq Z$ if $\bigcup_i R_i = S$.
We write $\Cov^{\Pi}_b(f)$ to denote the minimum size of a set of $\Pi$-rectangles that cover $f^{-1}(b)$. It turns out that this number is closely related to non-deterministic protocols: we define $\Non_b^{\Pi}(f) := \log_2 \Cov_b^{\Pi}(f)$. This is equal to the minimum number of bits needed by a two-party non-deterministic protocol for establishing that $f: X \times Y \to \{0,1\}$ evaluate to $b$, when one party is given access to the bits from $X$ and the other the bits from $Y$,						 see~\cite[Chapter~2.1]{KushilevitzNisan}. This is a communication complexity measure in the \emph{fixed partition} model.
  
In order to get a connection to d-SDNNF we need to instead look at the $\emph{best partition}$ model. We define the \emph{best-partition} non-deterministic communication complexity of $f$ as $\Non_1(f) := \min_{\Pi} \Non_1^{\Pi}(f)$, where the minimum is taken over all \emph{balanced} partitions. Similarly, we define $\Cov_b(f) : = \min_{\Pi} \Cov_b(f)$. We will also be interested in cases where the rectangles in a cover do not overlap. We say that $\Pi$-rectangles $R_1, \ldots, R_k$ 
\emph{partition} a set $S \subseteq Z$ if $\bigcup_i R_i = S$ 
and $R_i \cap R_j = \emptyset$ for all $i\neq j$. For~$b \in \{0,1\}$, 
the \emph{partition number} $\Par_b^{\Pi}(f)$ is the minimum number of 
$\Pi$-rectangles that partition $f^{-1}(b)$ and $\Par_b(f):= \min_{\Pi} \Par_b^{\Pi}$. 

We end this section by stating a connection between rectangular partitions of $f^{-1}(1)$ and d-SDNNF \cite{DBLP:conf/aaai/PipatsrisawatD10,bova2016knowledge}.
\begin{lemma} \label{lem: rectangle}
If $f\colon \{0,1\}^{n} \to \{0,1\}$ admits a d-SDNNF of size $s$, then $\Par_1(f) \le s$. Moreover, if $f$ admits an SDNNF of size $s$ then $\Cov_1(f) \le s$. 
\end{lemma}

\section{Proof of Theorems~\ref{thm: sep} and \ref{thm: main}} \label{sec: main}

\subsection{Proof Outline}

In \cite{goos2021lower} an analogue of Theorem~\ref{thm: main} is proved but for UFA. The approach goes via communication complexity in the fixed partition model. Our proof starts from this same piece of communication complexity. The difference is that we need to work in the best-partition model. It turns out we can lift the results from \cite{goos2021lower} to the best-partition model by adapting an ingenious construction of Knopp \cite{knop2017ips} based on the work of Segerlind \cite{DBLP:conf/coco/Segerlind08}.

The following is shown in the proof of \cite[Theorem 1]{goos2021lower} building on results from \cite{GLMWZ16} and \cite{DBLP:conf/focs/BalodisBG0K21}. 

\begin{theorem}[\cite{goos2021lower}] \label{thm: fixed_part}
For every $k \in \mathbb{N}$, there exists an integer $m = k^{O(1)}$, a Boolean function $g: \{0,1\}^m \to \{0,1\}$ and $\Pi$, a balanced partition of the inputs to $g$, such that the following properties hold.
\begin{enumerate}
\item $g$ is equivalent to an unambiguous $k$-DNF $\psi$ with $2^{\tO(k)}$ terms.
\item $\Non_0^{\Pi}(g)=\tOmega(k^2)$.
\end{enumerate}
\end{theorem}

One can show that $g$ admits a d-SDNNF of size $2^{\tO(k)}$. We would therefore like to show a lower bound on the size of an SDNNF equivalent to $\neg g$. However, we cannot use Lemma~\ref{lem: rectangle} since Theorem~\ref{thm: fixed_part} only gives lower bounds in the fixed-partition model. To get around this we transform $g$ into a new function $f$ which still admits a d-SDNNF of size $2^{\tO(k)}$ and with $\Non_0(f) \ge \Non_0^{\Pi}(g) = \tOmega(k^2)$.

\subsection{From Fixed Partition to Best Partition}

We now present the construction which allows us to build $f$. This is almost the same as that given in \cite{knop2017ips} based on the work from \cite{DBLP:conf/coco/Segerlind08}. The following things are different (1) we are now working with formulas in DNF rather than CNF, (2) we now want to ensure that the construction transforms an unambiguous DNF into an unambiguous DNF and (3) the notion of balancedness we use is different. We end up with the following result. 

\begin{theorem} \label{thm: fixed_to_best}
Let $\psi$ be an unambiguous n-variable $k$-DNF with $\ell$ terms. Then there 
exists an unambiguous $O(n^2)$ variable $O(kn)$-DNF $\psi'$ with $O(\ell n^{k+4})$ 
terms such that for $\delta \in \{0,1\}$ and any balanced partition $\Pi$ of 
the variables of $\psi$, $\Non_{\delta}(\psi') \ge \Non_{\delta}^{\Pi}(\psi)$.
\end{theorem}

We next show how this implies Theorem~\ref{thm: main}.

\begin{proof}[Proof of Theorem~\ref{thm: main}]
Fix some $k \in \mathbb{N}$ and let $g: \{0,1\}^m \to \{0,1\}$ be the 
function from Theorem~\ref{thm: fixed_part}. Then we know there is some 
equivalent unambiguous $k$-DNF $\psi$ with $2^{\tO(k)}$ terms. Let $f 
\equiv \psi'$. By Theorem~\ref{thm: fixed_to_best} and since $n = k^{O(1)}$, 
$\psi'$ has $2^{\tO(k)}$ terms. Then we can form a d-SDNNF equivalent to 
$f$ as follows. Fix any v-tree $T$ over $\var(\psi')$. Then since every term of $\psi'$ is a conjunction of $O(km)$ literals, they all admit a d-DNNF respecting $T$ of size $O(km)$. By taking the disjunction of all such d-DNNF we get a d-DNNF for $\psi'$ respecting $T$ of size $2^{\tO(k)}:= n$. Here determinism follows as $\psi'$ is unambiguous.

But also by Theorem~\ref{thm: fixed_part}, $\Non_0^{\Pi}(g) = \tOmega(k^2)$. So applying Theorem~\ref{thm: 
fixed_to_best} we obtain that $\Non_0(f) \ge \Non_0^{\Pi}(g) = \tOmega(k^2)$. So $ 
\Cov_0(f) = 2^{\Non_0(f)} = 2^{\tOmega(k^2)}$. Therefore, by 
Lemma~\ref{lem: rectangle}, any SDNNF for $\neg f$ has size at least 
$2^{\tOmega(k^2)} = n^{\tOmega(\log n)}$. 
\end{proof}

It only remains to prove Theorem~\ref{thm: fixed_to_best}. 
\subsection{Idea of the Construction}

One way to try and transfer bounds from the 
fixed partition 
model to the best partition model is to extend a function to include a 
permutation as part of the input. A naive way of doing this is as 
follows. Let $\psi(x_1, \dots, x_n)$ be a propositional formula. Associate a 
binary string of length $m:= \log_2(n!)$ to each 
of the permutations of these variables and for $\sigma \in S_n$  write $
\rep(\sigma)_i$ for the $i$th bit of the string associated with $\sigma$. We define
\begin{align*}
&\perm_{S_n}(\psi)(z_1, \dots, z_m, x_1, \dots, x_n) \equiv \\ &\bigwedge_{\sigma 
\in S_n} \left( \bigwedge_{i=1}^m (z_i = \rep(\sigma)_i) \rightarrow 
\psi(x_{\sigma(1)}, \dots, x_{\sigma(n)}) \right).
\end{align*}

What is the idea? Let $\Pi = 
(X, Y)$ be a balanced partition of $\var(\psi)$ and $\Gamma = (X',Y')$ be any balanced partition of $\var(\perm_{S_n}(\psi))$ such that $|X| = |\{x_i \, \mid \, x_i \in X'\}|$. Suppose we have a non-deterministic protocol for $\perm_{S_n}(\psi)$ under partition $\Gamma$ of communication complexity $c$. We now define a non-deterministic protocol for $
\psi$ under partition $\Pi$.
Let $\sigma \in S_n$ be the permutation such that $X = \{\sigma(x_i) \, \mid \, 
x_i \in X'\}$. Suppose we have an input for $\psi$, $x_i \to a_i$, $i \in [n]$. Run the protocol for $\perm_{S_n}(\psi) : X' \times Y' \to \{0,1\}$ on the input where $z_i \to \rep(\sigma)_i$ and $x_{\sigma(i)} \to a_i$. By definition, this is a protocol for $\psi : X' \times Y' \to \{0,1\}$ of communication complexity $c$. Therefore, $\Non^{\Pi}_1(\perm_{S_n}(\psi)) \ge \Non^{\Gamma}_1(\psi)$. But there is an issue: the size of $\perm(S_n)(\psi)$ is exponential in $n$. So our proof wouldn't go through with this construction because there is no reason to think that $\psi$ admitting a small d-SDNNF implies that $\perm_{S_n}(\psi)$ admits such a representation.\footnote{Another issue is that this argument only gives lower bounds for partitions where $|X| = |\{x_i \, \mid \, x_i \in X'\}|$.}

The obvious solution is to consider a smaller set of permutations. But then the $\sigma$ such that $X = \{\sigma(x_i) \, \mid \, 
x_i \in X'\}$ might not be in our set. To get around this we, following \cite{knop2017ips}, \emph{first} add copies of variables to our original formula and \emph{then} permute these new variables. This effectively increases the number of permutations we can reach. 

\subsection{The Construction}

\subsubsection{Step 1: Making Copies of Variables}

Let $\psi$ be an unambiguous DNF formula on $n$ variables. We replace every occurrence of variable $x_i$ by a disjunction of $m$ fresh variables $\bigvee_{j \in [m]} y_{i,j}$, where $m = cn$, for some sufficiently large constant $c$. 
Denote the subformula obtained from a term $C$ of $\psi$ by $C^{\vee}$. 
The resulting formula is \emph{not} a DNF. To change this first we use  
distributivity to expand out our formula into a DNF $\phi$. If a term of 
$\phi$ is the result of expanding out $C^{\vee}$ we say it is 
\emph{derived} from $C$. As this DNF may not be unambiguous we need an 
extra step not in \cite{knop2017ips}. Take a term $C$ of 
$\phi$. For each positive literal $y_{i,j}$ in $C$ add conjuncts $\neg 
y_{i,j'}$ for every $j'\neq j$ to get a term $C^{u}$. If $C$ is derived 
from $D$ we also say that $C^{u}$ is derived from $D$. Repeat this for 
every  term of $\phi$. We denote the resulting DNF by $\psi^{\vee}$. 

\begin{lemma}
If $\psi$ is an unambiguous $n$-variable $k$-DNF with $\ell$ terms then $\psi^{\vee}$ is an $O(n^2)$-variable unambiguous $O(kn)$-DNF with $O(\ell n^{k})$ terms.
\end{lemma}

\begin{proof}
To see that $\psi^{\vee}$ is unambiguous suppose some assignment $\alpha$ 
of the variables of $\psi^{\vee}$ satisfies a term $C$. Define the 
assignment $\beta = \beta(\alpha)$ on the variables of $\psi$ by $
\beta(x_i) = 1$ if and only if $\alpha(\bigvee_{j \in [m]} y_{i,j}) = 1$. 
We know that $C$ was derived from some $D$. Then clearly $\beta$ satisfies 
$D$. Since $\psi$ is unambiguous this is the unique term satisfied by $
\beta$. Therefore, if $\alpha$ satisfies some $C'$ then $C'$ must also be 
derived from $D$. We have that there are $I_p, I_n \subseteq [n]$ with $|I_p 
\cup I_n| = O(k)$  such that $D^{\vee}$ is equal to 
\begin{equation*}
\bigwedge_{i \in I_p} \bigvee_{j=1}^m y_{i,j} \wedge \bigwedge_{i \in I_n} \neg \bigvee_{j=1}^m y_{i,j}
\equiv 
\bigwedge_{i \in I_p} \bigvee_{j=1}^m y_{i,j} \wedge \bigwedge_{i \in I_n} \bigwedge_{j=1}^m \neg y_{i,j}
\end{equation*}

It follows that each term of $\psi^{\vee}$ deriving from $D$ is of the form
\begin{equation*} \label{eq: term_form}
\bigwedge_{i \in I_p} y_{i, j_i} \wedge \bigwedge_{j \neq j_i} \neg y_{i,j} \wedge \bigwedge_{i \in I_n} \bigwedge_{j=1}^m \neg y_{i,j}
\end{equation*}
where $(j_i)_{i \in I_p}$ is some sequence of elements of $[m]$. It is therefore easy to see that $C = C'$ and so $\psi^{\vee}$ is unambiguous. Moreover, by the above each term contains $O(km) = O(kn)$ variables. Finally, by construction, there are at most $m^k$ terms of $\phi^{\vee}$ derived from each term of $\phi$.
\end{proof}

\subsubsection{Step 2: Adding permutations}

Let $|\mathsf{var}(\psi^{\vee})| =n'$. For notational convenience we denote $y_{i,j} := v_{im+j}$ and $\mathsf{var}(\psi^{\vee})$ by $V$. For simplicity\footnote{The general case is not much more difficult: we just add extra `dummy' variables until we reach a power of two. This does not change anything, other than making the notation slightly messier, see \cite[Theorem 4.2]{knop2017ips}.} assume $n' = 2^t$ for some integer $t$. Let $\mathbb{F}$ be the unique field of order $n'$ and $\mathcal{P}$ be the set of mapping on $\mathbb{F}$ with $x \to ax +b$ , $a,b \in \mathbb{F}$ and $a \neq 0$. The reason for using $\mathcal{P}$ is that it is a set of independent permutations in the following sense.

\begin{lemma} \label{lem: indperm} \cite{wegman1981new}
Every mapping in $\mathcal{P}$ is a permutation and $|\mathcal{P}| = n' \cdot (n'-1)$. Moreover, for any $a,b,c,d \in [n']$ with $a \neq b$ and $c \neq d$,
\[
\underset{\sigma \in \mathcal{P}}{\mathsf{Pr}} [\sigma(a) = c, \,\sigma(b) = d] = \frac{1}{|\mathcal{P}|}
\]
\end{lemma}

Elements of $\mathcal{P}$ may be represented by binary strings of length $2t$ such that the first $t$ bits are not all zero; the $i$th bit of the representation of $\sigma \in \mathcal{P}$  is denoted $\rep(\sigma)_i$. 

Let $C$ be a term of $\psi^{\vee}$ with $C = \bigwedge_{i \in I} a_i$, for some $a_i \in \{v_i,  \neg v_i\}$. Then for each $\sigma \in \mathcal{P}$ and $C$ a term of $\psi^{\vee}$ we define a 
term
\[
\perm_{\pi}(C):= \bigwedge_{i=1}^{2t} \left(z_i = \rep(\sigma)_i \right) \wedge \bigwedge_{i \in I} a_{\sigma(i)}
\]
To form $\psi'$ we take the disjunction of every $\perm_{\pi}(C)$ with $\pi \in \mathcal{P}$, $C$ a term of $\psi^{\vee}$. Note, that if $\phi^{\vee}$ is unambiguous so is $\psi'$ and that $|\mathcal{P}| = O(n^4)$. The following is immediate.

\begin{lemma} \label{lem: well_def}
If $\psi$ is a $n$-variable unambiguous $k$-DNF with $\ell$ terms then $\psi'$ is a $O(n^2)$-variable unambiguous $O(kn)$-DNF with $O(\ell n^{k+4})$ terms.
\end{lemma}

\subsection{Proof of Theorem~\ref{thm: fixed_to_best}}

\begin{proof}
Fix two arbitrary balanced partitions $\Pi = (X, Y)$ and $\Gamma = (X', Y')$ of the variables of $\psi$ and $\psi'$ respectively. It is enough to show that if there is a non-deterministic protocol for $\psi'$ (resp. $\neg \psi'$) under partition $\Gamma$ then there is a non-deterministic protocol for $\psi$ (resp. $\neg \psi$) under partition $\Pi$ with the same communication complexity. The key is the following.

\begin{claim} \label{claim: perm}
There is a permutation $\sigma \in \mathcal{P}$, such that for any $i\in [n]$ and $k\in \{0,1\}$, there is a $j$ such that $y_{i,j}$ is mapped to a variable from $\Pi_k$ by $\sigma$.
\end{claim}

Assume the claim and let $\sigma$ be such a permutation. Let $v_{r(i,k)}$ denote some variable $y_{i,j}$ that is mapped to an element of $\Pi_k$ by the permutation $\sigma$. Now consider the following protocol for $\psi$ under the partition $\Pi$. On input $x_1 \to a_1, \dots, x_n \to a_n$ it runs the protocol for $\psi'$ under partition $\Gamma$ on the input such that
\begin{itemize}
\item the $z_i$ encode $\sigma$, i.e., $z_i$ is set to $\rep(\sigma)_i$,
\item if $x_i \in \Gamma_k$, $v_{r(i,k)}$ is set to be equal to $a_i$ and
\item  every other variable in $V$ is set to be zero.
\end{itemize}
By the construction of $\psi'$ this is a correct protocol for $\psi$. The case for $\neg \psi$ is identical except now $\neg \psi'$ plays the role of $\psi'$.

Claim~\ref{claim: perm} follows by a relatively simple probabilistic argument: the two key tools are Chebyshev's inequality and Lemma~\ref{lem: indperm}. The proof is almost identical to \cite[Theorem 4.2.]{knop2017ips}. The only difference is that we need to use  our more relaxed notion of balancedness but an inspection of the proof shows that everything goes through. 
\end{proof}

Theorem~\ref{thm: sep} follows almost immediately.

\begin{proof}[Proof of Theorem~\ref{thm: sep}]
Fix some $n \in \mathbb{N}$, let $f$ be the function given by Theorem~\ref{thm: main} and let $\C$ be an SDD equivalent to $f$. Then we may complement this SDD to get an SDD for $\neg f$ of size polynomial in $|\C|$. Since SDD is a subset of d-SDNNF by Theorem~\ref{thm: main} this must be of size $n^{\tO(\log(n)}$ and the result follows.
\end{proof}

\subsection{Disjunction and Existential Quantification}
In the appendix, we also prove the following theorem.
\begin{theorem}[Disjunction] \label{thm: union}
For every $n \in \N$, there exists Boolean functions $f,g$ sharing a common domain with the following properties.
\begin{enumerate}
\item There is a v-tree $T$ such that $f$ and $g$ both admit d-DNNFs of size $n$ respecting $T$.
\item Any d-SDNNF equivalent to $f \vee g$ has size $n^{\tOmega(\log n)}$.
\end{enumerate}
\end{theorem}

The proof follows the same structure as that given above. In fact, 
Theorem~\ref{thm: union} almost implies\footnote{Note that in 
Theorem~\ref{thm: main} we prove lower bounds on the size of SDNNF 
equivalent to $\neg f$. Using  Theorem~\ref{thm: union} we could only get a 
lower bound on d-SDNNF. Still, this is, arguably, the most important part of 
the result. To go from Theorem~\ref{thm: union} to the result on negation  
one uses the equivalence $ f \vee g \equiv \neg(\neg f \wedge \neg g)$ and 
the tractability of $\wedge$ for d-SDNNF \cite{DBLP:conf/aaai/PipatsrisawatD08}.} Theorem~\ref{thm: main}; however we focused our 
expositional on negation for pedagogical reasons and since this is related 
to a long-standing open question. Namely, does (unstructured) d-DNNF admit 
polynomial time negation \cite{DBLP:journals/jair/DarwicheM02}? We discuss 
this question in more detail in the conclusion. We also obtain the 
following corollary.

\begin{corollary}[Existential Quantification] \label{thm: ex}
For every $n \in \N$, there exists a set $X$, a Boolean function $f: \{0,1\}^X \to \{0,1\}$ and a variable $x \in X$, such that
\begin{enumerate}
\item $f$ admits a d-SDNNF of size $n$ and
\item any d-SDNNF equivalent to $\exists xf$ has size $n^{\tOmega(\log n)}$.
\end{enumerate}
\end{corollary}

\begin{proof} Let $f,g, T$ be as in the statement of Theorem~\ref{thm: union}. Then there are d-DNNFs $\C_f$ and $\C_g$, for $f$ and $g$ respectively, which both have size $n$ and respect $T$. Let $x$ be a fresh variable and form a new circuit $\C$ with 
\[ 
\langle \C \rangle = (x \wedge \langle \C_f \rangle ) \vee (\neg x \wedge \langle \C_g \rangle)
\]
The source is a $\vee$-node $g$ which is deterministic since any element of $\sat(g_{\ell})$ must map $x \to 1$ and any element of $\sat(g_r)$ must map $x \to 0$. Take $T$ and form a v-tree $T'$ by adding two fresh nodes $r, t$ such that $r$ has children $t$ and the root of $T$. Then $\C$ is a d-DNNF of size $O(n)$ respecting $T'$. Moreover, $\exists x f_C \equiv f \vee g$. The result follows by Theorem~\ref{thm: union}.
\end{proof}

\section{Lifting results to ACs} \label{sec: AC}
So far we have been focussed on representations of Boolean functions; now we switch gears and look at representations of real-valued polynomials: arithmetic circuits (AC). These are defined similarly to NNF except now internal nodes are labelled by $+$ and $\times$ and any real number may be a constant. 

\begin{definition} \label{def: AC}
%An arithmetic circuit (AC) is a directed acyclic graph with a unique source whose internal nodes are fan-in 2 $+$ and $\times$ gates and whose inputs are each labelled by a real number, a variable $x$ or a negated variable $\neg x$. 
An arithmetic circuit (AC) is a vertex-labelled directed acyclic graph with a unique source such that every internal node is a fan-in two $+$- or $\times$-node and whose leaves are each labelled by 0, 1, a variable $x$ or a negated variable $\neg x$.

\end{definition}

As in the case of NNFs, we associate a set of variables $\dom(\C)$ to each arithmetic circuit which contains every variable occurring in the circuit. Then we associate to $\C$ a function $f_{\C}: \{0,1\}^{\dom(\C)} \to \mathbb{R}$ as follows. On input $\underline{x}$, replace each variable $x$ occurring positively in $\C$ with $\underline{x}(x)$ and each variable occurring negatively with $1-\underline{x}(x)$. Then evaluate the circuit bottom up in the obvious way: the output is $f_{\C}(\underline{x})$. If we expand out the circuit we get a formula in $\var(\C)$ (with possible negations). We denote this by $\langle \C \rangle$ and identify this expression with the function $f_{\C}$.

Since in many cases, ACs are used in the context of probabilistic reasoning it often makes sense to restrict our attention to ACs which output non-negative polynomials; call this fragment \emph{positive} AC, denoted $\ACp$. This can be enforced syntactically by insisting that every constant is non-negative, we call this fragment \emph{monotone} AC, denoted $\AC$. This fragment includes many well-studied classes such as PSDD \cite{DBLP:conf/kr/KisaBCD14} and Sum Product Networks (SPN) \cite{DBLP:conf/iccvw/PoonD11}. Moreover, \cite{DBLP:conf/kr/ColnetM21} made the following observations connecting $\AC$ with NNF.

Let $\C$ be am $\AC$. We form an NNF circuit $\phi(\C)$ with the same underlying directed graph as $\C$ by relabelling each node as follows:
\begin{itemize}
\item Leaves: nodes labelled by a variable, a negated variable or the constant 0 are unchanged. Otherwise, change the label to the constant 1.
\item Internal node: change $+$-nodes to $\vee$-nodes and $\times$-nodes to $\wedge$-nodes.
\end{itemize}
We will use the following result.
\begin{lemma}[{\cite[Proposition 2]{DBLP:conf/kr/ColnetM21}}]
\label{lem: AC}
Let $\mathsf{C}_1$ and $\mathsf{C}_2$ be sets of $AC_m$. Then $\mathsf{C}_1 \le \mathsf{C}_2$ implies that $\phi(\mathsf{C}_1) \le \phi(\mathsf{C}_2)$. 
\end{lemma} 

Here $\le$ is defined exactly as for subsets of NNF.  Write $\supp(\C)$ to denote the support of $\C$, i.e., the set of inputs for which $f_{\C}$ is non-zero. We can also lift the definitions of decomposability, determinism and structuredness to AC, by replacing the role of $\wedge$ with $\times$, $\vee$ with $+$ and $\sat$ with $\supp$ in the definitions. Let dSD-$\AC$ denote the deterministic, structured, decomposable, monotone ACs. This is the $\AC$ analogue to d-SDNNF. We next define the $\AC$ analogue to SDD.

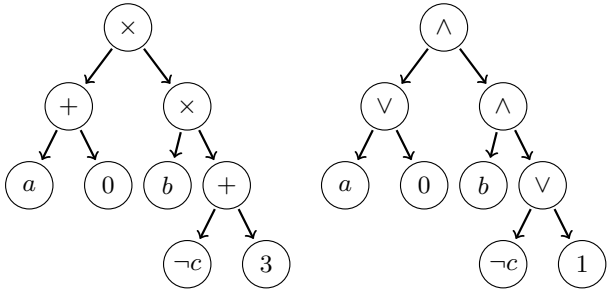
\begin{figure} 
\centering
\scalebox{1.05}{
\begin{tikzpicture}
 [wire/.style={thick, ->, shorten <= 0.3mm, shorten >= 0.5mm},
  gwire/.style={thick, ->, shorten <= 0.3mm, shorten >= 0.5mm, color=green},
   edge/.style={-stealth, shorten <= 0.3mm, shorten >= 0.5mm},
   inputgate/.style={inner sep=1pt, minimum size=4mm},
   gate/.style={draw, circle, inner sep=1pt, minimum size=6mm, font=\small},
   vertex/.style={draw, circle, fill=white, inner sep=2pt, minimum size=20mm}] 
    \begin{scope}[]
   
     \node[gate] (g1) at (-1.25,1) {$a$};
     \node[gate] (g2) at (-1.25+1,1) {$0$};
      \node[gate] (g4) at (-0.75,2) {$+$};
      \node[gate] (g5) at (0.75,2) {$\times$};
      \node[gate] (g6) at (-0,3) {$\times$};

      \node[gate] (d) at (0.5,1) {$b$};
      \node[gate] (e) at (1.25,1) {$+$};
      
	     \node[gate] (w1) at (0.75,0) {$\neg c$};
      \node[gate] (w2) at (1.75,0) {$3$};

      \draw[wire] (g4) -- (g2);
      \draw[wire] (g6) -- (g4);
      \draw[wire] (g6) -- (g5);
      \draw[wire] (g5) -- (d);
      \draw[wire] (g5) -- (e);
      \draw[wire] (g4) -- (g1);

       \draw[wire] (e) -- (w1);
         \draw[wire] (e) -- (w2);
 \end{scope}   
 
     \begin{scope}[xshift=4cm]
   
     \node[gate] (g1) at (-1.25,1) {$a$};
     \node[gate] (g2) at (-1.25+1,1) {$0$};
      \node[gate] (g4) at (-0.75,2) {$\vee$};
      \node[gate] (g5) at (0.75,2) {$\wedge$};
      \node[gate] (g6) at (-0,3) {$\wedge$};

      \node[gate] (d) at (0.5,1) {$b$};
      \node[gate] (e) at (1.25,1) {$\vee$};
      
	     \node[gate] (w1) at (0.75,0) {$\neg c$};
      \node[gate] (w2) at (1.75,0) {$1$};

      \draw[wire] (g4) -- (g2);
      \draw[wire] (g6) -- (g4);
      \draw[wire] (g6) -- (g5);
      \draw[wire] (g5) -- (d);
      \draw[wire] (g5) -- (e);
      \draw[wire] (g4) -- (g1);

       \draw[wire] (e) -- (w1);
         \draw[wire] (e) -- (w2);
 \end{scope}   
\end{tikzpicture}
}
\caption{(left) A monotone arithmetic circuit, $\C$. Note that $\langle \C \rangle = (a+0) \times (b \times (\neg c +3))$. To evaluate $f_{\C}$ for $a=1, b=1, c=0$ we compute $(1+0) \times (1 \times ((1-0) +3)) = 4$. (right) $\phi(\C)$, an NNF. Note that $\supp(\C) = \sat(\phi(\C))$.}
\end{figure}

\begin{definition} \label{def: p-decomp}
Let $f: \{0,1\}^X \to \mathbb{R}^{+}$ be a positive polynomial and  $X,Y \subseteq Z$ be disjoint sets of variables. Suppose
\[
f = \sum_{i=1}^n \alpha_i \times (p_i(X) + s_i(Y))
\]  
 where each $\alpha_i > 0$ and $\sum_{i=1}^n \alpha_i = 1$. Then $\{(p_1, s_1, \alpha_1), \dots, (p_n, s_n, \alpha_n)\}$ is an $X$ $p$-decomposition for $f$ if $\sum_{i=1}^n p_i \equiv 1$, $p_i \times p_j \equiv 0$ for all $i\neq j$ and $p_i \not \equiv 0$ for all $i$.
\end{definition}

The idea is that we take an $X$ decomposition, put a distribution on the disjuncts and then replace  $\vee$ with $+$ and $\wedge$ with $\times$. We call the $\alpha_i$ \emph{parameters}. A PSDD is defined analogously to an SDD but with $+$, $\times$ and $X$ $p$-decomposition replacing the roles of $\vee, \wedge$ and $X$ decomposition. 

\begin{definition} \label{def: PSSD}
 Let $T$ be a v-tree over variables $Z$ with root $t$.  Then $\C$ is a PSDD respecting $T$ if it is an $\AC$ with one of the following forms.
 \begin{itemize}
 \item $\C$ consists of a single node labelled by a constant, $x$ or $\neg x$, where $x \in Z$.
 \item The source of $\C$ is a $+$ node $g$ such that:
 \begin{enumerate}
 \item $\langle \C \rangle = \sum_{i=1}^n \alpha_i \times (p_i(X) + s_i(Y))$ is an $X$ p-decomposition for $f_{\C}$,
 \item $X \subseteq \var(t_{\ell})$, $Y \subseteq \var(t_r)$ and
 \item if $h \in \C$ with $\langle \C(h) \rangle = p_i(X)$ (resp. $s_i(Y)$) for some $i$ then $\C(h)$ is a PSDD that respects $t_{\ell}$ (resp. $t_r$).
\end{enumerate}  
 \end{itemize}
 A PSDD is a PSDD respecting some v-tree.
\end{definition}  

The following is now almost immediate.

\begin{corollary} \label{cor: ACsep}
 dSD-$\AC$ $<$ PSDD.
\end{corollary}

\begin{proof}
First, observe that $\phi(\textrm{dSD-}\AC) =$ d-SDNNF. It is not quite true that $\phi(\textrm{PSDD}) = $ SDD. However, for any $\C \in \phi(\textrm{PSDD})$, if we propagate away constants corresponding to parameters from $\phi(\C)$ we get an equivalent SDD of smaller size, $\C'$.  Therefore, $\phi(\textrm{PSDD}) \ge$ SDD. By Theorem~\ref{thm: sep}, $\phi(\textrm{dSD-}\AC) =$ d-SDNNF $<$ SDD $\le \phi(\textrm{PSDD})$ and so by Lemma~\ref{lem: AC} the result follows.
\end{proof}

Since adding two positive functions yields a positive function we also lift Theorem~\ref{thm: union}.

\begin{corollary} \label{cor: add}
For every $n \in \N$, there exists positive polynomials $f$ and $g$ which both  admit a dSD-$\AC$ of size $n$ and such that any dSD-$\ACp$ equivalent to $f + g$ has size $n^{\tOmega(\log n)}$.
\end{corollary}

\begin{proof}
Let $f$ and $g$ be as in the statement of Theorem~\ref{thm: union}. Take 
any d-SDNNF equivalent to $f$. Then if we change every $\vee$ to a $+
$ and every $\wedge$ to a $\times$ we get a dSD-$\AC$ of the same size 
which is equivalent to $f$ viewed as a positive polynomial. The same is true for $g$.  
Let $\C$ be a dSD-$\ACp$ equivalent to $f+g$, again viewed as a positive 
polynomial. If we switch the sign of every negative constant then by 
\cite[Lemma 10]{DBLP:conf/kr/ColnetM21} we get an equivalent dSD-$\AC$, 
call it $\C_m$. Then $\phi(\C_m)$ is a d-SDNNF for $f \cup g$. Therefore, by 
Theorem~\ref{thm: union}, $|\C| = |\C_m|  =n^{\tOmega(\log n)}$. 
\end{proof}

\section{Conclusion and Open Problems} \label{sec: con}

We have shown that d-SDNNF does not admit polynomial time complementation, disjunction or existential quantification and that it is more succinct than SDD. Therefore, there is a trade-off between succinctness and supported transformations in choosing one representation over the other. We have shown a quasi-polynomial separation but have not ruled out the possibility that the gap could in fact be exponential. 

A tantalising open problem, first raised over twenty years ago \cite{DBLP:journals/jair/DarwicheM02}, is whether d-DNNF is closed under complementation. We have solved a restricted form of this problem and one could attempt the general case using similar methods. Just as the size of d-SDNNF is related to the best partition communication complexity, the size of d-DNNF is related to \emph{multi-partition} communication complexity \cite{DBLP:conf/aaai/Bova16}. However, adapting the methods from this paper to this setting still appears to be a daunting task. 

\appendix

\section{Proof of Theorem~\ref{thm: union}}

We will now be concerned with \emph{unambiguous} protocols. These are non-deterministic protocols which have at most one accepting path on every input. Given a function $f: \{0,1\}^n \to \{0,1\}$ and a balanced partition of the inputs $\Pi = (X,Y)$ we write $\Un^{\Pi}_b(f):= \log_2 \Par^{\Pi}_b(f)$; this is equal to the minimum number of bits in an unambiguous protocol for establishing that $f: X \times Y \to \{0,1\}$ evaluates to $b$. The proof follows a similar structure to Theorem~\ref{thm: main}. This time we start from the following result which is shown in the proof of \cite[Theorem 2]{goos2021lower}. 

\begin{theorem} \label{thm: fixed_or}
For every $k \in \mathbb{N}$,
 there exists $n = k^{O(1)}$, Boolean function $f, 
g: \{0,1\}^n \to \{0,1\}$ and a balanced partition $\Pi$ such that the 
following properties hold.
\begin{enumerate}
\item $f,g$ have equivalent unambiguous $k$-DNFs $\psi, \phi$ respectively with $2^{\tO(k)}$ terms.
\item $\Un_1^{\Pi}(f \cup g)=\tOmega(k^2)$.
\end{enumerate}
\end{theorem}

So let $f,g, \psi, \phi$ be as above. Then, by Theorem~\ref{thm: fixed_to_best} and the same argument as in Theorem~\ref{thm: main}, $\psi'$ and $\phi'$ have equivalent d-SDNNFs of size $2^{\tO(k)}$. Therefore, by Lemma~\ref{lem: rectangle}, it is enough to show that $\Un_1(\psi' \cup \phi') = \tOmega(k^2)$.

So fix an arbitrary balanced partition $\Gamma$ of the variables of $\psi' \cup \phi'$. We want to show that if there is an unambiguous protocol for $ \psi' \cup \phi'$ under partition $\Gamma$ then there is an unambiguous protocol for $f \cup g$ under partition $\Pi$ with the same communication complexity. But this follows in essentially the same way as in the proof of Theorem~\ref{thm: main} by applying Claim~\ref{claim: perm}. Therefore, $ \Un_1(\psi' \cup \phi') \ge \Un_1^{\Pi}(f \cup g) = \tOmega(k^2)$.

\section*{Acknowledgements}

The author would like to thank Stefan Mengel for his many helpful suggestions pertaining to the topics of this paper. 

\bibliographystyle{named}
\bibliography{ijcai24}

\begin{thebibliography}{}

\bibitem[\protect\citeauthoryear{Amarilli \bgroup \em et al.\egroup
  }{2020}]{DBLP:journals/mst/AmarilliCMS20}
Antoine Amarilli, Florent Capelli, Mika{\"{e}}l Monet, and Pierre Senellart.
\newblock Connecting knowledge compilation classes and width parameters.
\newblock {\em Theory Comput. Syst.}, 64(5):861--914, 2020.

\bibitem[\protect\citeauthoryear{Balodis \bgroup \em et al.\egroup
  }{2021}]{DBLP:conf/focs/BalodisBG0K21}
Kaspars Balodis, Shalev Ben{-}David, Mika G{\"{o}}{\"{o}}s, Siddhartha Jain,
  and Robin Kothari.
\newblock Unambiguous dnfs and alon-saks-seymour.
\newblock In {\em 62nd {IEEE} Annual Symposium on Foundations of Computer
  Science, {FOCS} 2021, Denver, CO, USA, February 7-10, 2022}, pages 116--124.
  {IEEE}, 2021.

\bibitem[\protect\citeauthoryear{Beame and Liew}{2015}]{DBLP:conf/uai/BeameL15}
Paul Beame and Vincent Liew.
\newblock New limits for knowledge compilation and applications to exact model
  counting.
\newblock In Marina Meila and Tom Heskes, editors, {\em Proceedings of the
  Thirty-First Conference on Uncertainty in Artificial Intelligence, {UAI}
  2015, July 12-16, 2015, Amsterdam, The Netherlands}, pages 131--140. {AUAI}
  Press, 2015.

\bibitem[\protect\citeauthoryear{Bollig and
  Farenholtz}{2021}]{DBLP:journals/mst/BolligF21}
Beate Bollig and Martin Farenholtz.
\newblock On the relation between structured d-dnnfs and sdds.
\newblock {\em Theory Comput. Syst.}, 65(2):274--295, 2021.

\bibitem[\protect\citeauthoryear{Bova and
  Szeider}{2017}]{DBLP:conf/pods/BovaS17}
Simone Bova and Stefan Szeider.
\newblock Circuit treewidth, sentential decision, and query compilation.
\newblock In Emanuel Sallinger, Jan~Van den Bussche, and Floris Geerts,
  editors, {\em Proceedings of the 36th {ACM} {SIGMOD-SIGACT-SIGAI} Symposium
  on Principles of Database Systems, {PODS} 2017, Chicago, IL, USA, May 14-19,
  2017}, pages 233--246. {ACM}, 2017.

\bibitem[\protect\citeauthoryear{Bova \bgroup \em et al.\egroup
  }{2016}]{bova2016knowledge}
Simone Bova, Florent Capelli, Stefan Mengel, and Friedrich Slivovsky.
\newblock Knowledge compilation meets communication complexity.
\newblock In Subbarao Kambhampati, editor, {\em Proceedings of the Twenty-Fifth
  International Joint Conference on Artificial Intelligence, {IJCAI} 2016, New
  York, NY, USA, 9-15 July 2016}, pages 1008--1014. {IJCAI/AAAI} Press, 2016.

\bibitem[\protect\citeauthoryear{Bova}{2016}]{DBLP:conf/aaai/Bova16}
Simone Bova.
\newblock Sdds are exponentially more succinct than obdds.
\newblock In Dale Schuurmans and Michael~P. Wellman, editors, {\em Proceedings
  of the Thirtieth {AAAI} Conference on Artificial Intelligence, February
  12-17, 2016, Phoenix, Arizona, {USA}}, pages 929--935. {AAAI} Press, 2016.

\bibitem[\protect\citeauthoryear{Bryant}{1986}]{DBLP:journals/tc/Bryant86}
Randal~E. Bryant.
\newblock Graph-based algorithms for boolean function manipulation.
\newblock {\em {IEEE} Trans. Computers}, 35(8):677--691, 1986.

\bibitem[\protect\citeauthoryear{Darwiche and
  Marquis}{2002}]{DBLP:journals/jair/DarwicheM02}
Adnan Darwiche and Pierre Marquis.
\newblock A knowledge compilation map.
\newblock {\em J. Artif. Intell. Res.}, 17:229--264, 2002.

\bibitem[\protect\citeauthoryear{Darwiche}{2001a}]{DBLP:journals/jacm/Darwiche01}
Adnan Darwiche.
\newblock Decomposable negation normal form.
\newblock {\em J. {ACM}}, 48(4):608--647, 2001.

\bibitem[\protect\citeauthoryear{Darwiche}{2001b}]{DBLP:journals/jancl/Darwiche01}
Adnan Darwiche.
\newblock On the tractable counting of theory models and its application to
  truth maintenance and belief revision.
\newblock {\em J. Appl. Non Class. Logics}, 11(1-2):11--34, 2001.

\bibitem[\protect\citeauthoryear{Darwiche}{2011}]{DBLP:conf/ijcai/Darwiche11}
Adnan Darwiche.
\newblock {SDD:} {A} new canonical representation of propositional knowledge
  bases.
\newblock In Toby Walsh, editor, {\em {IJCAI} 2011, Proceedings of the 22nd
  International Joint Conference on Artificial Intelligence, Barcelona,
  Catalonia, Spain, July 16-22, 2011}, pages 819--826. {IJCAI/AAAI}, 2011.

\bibitem[\protect\citeauthoryear{de Colnet and
  Mengel}{2021}]{DBLP:conf/kr/ColnetM21}
Alexis de~Colnet and Stefan Mengel.
\newblock A compilation of succinctness results for arithmetic circuits.
\newblock In Meghyn Bienvenu, Gerhard Lakemeyer, and Esra Erdem, editors, {\em
  Proceedings of the 18th International Conference on Principles of Knowledge
  Representation and Reasoning, {KR} 2021, Online event, November 3-12, 2021},
  pages 205--215, 2021.

\bibitem[\protect\citeauthoryear{Gogic \bgroup \em et al.\egroup
  }{1995}]{DBLP:conf/ijcai/GogicKPS95}
Goran Gogic, Henry~A. Kautz, Christos~H. Papadimitriou, and Bart Selman.
\newblock The comparative linguistics of knowledge representation.
\newblock In {\em Proceedings of the Fourteenth International Joint Conference
  on Artificial Intelligence, {IJCAI} 95, Montr{\'{e}}al Qu{\'{e}}bec, Canada,
  August 20-25 1995, 2 Volumes}, pages 862--869. Morgan Kaufmann, 1995.

\bibitem[\protect\citeauthoryear{G{\"{o}}{\"{o}}s \bgroup \em et al.\egroup
  }{2016}]{GLMWZ16}
Mika G{\"{o}}{\"{o}}s, Shachar Lovett, Raghu Meka, Thomas Watson, and David
  Zuckerman.
\newblock Rectangles are nonnegative juntas.
\newblock {\em {SIAM} J. Comput.}, 45(5):1835--1869, 2016.

\bibitem[\protect\citeauthoryear{G{\"{o}}{\"{o}}s \bgroup \em et al.\egroup
  }{2022}]{goos2021lower}
Mika G{\"{o}}{\"{o}}s, Stefan Kiefer, and Weiqiang Yuan.
\newblock Lower bounds for unambiguous automata via communication complexity.
\newblock In Mikolaj Bojanczyk, Emanuela Merelli, and David~P. Woodruff,
  editors, {\em 49th International Colloquium on Automata, Languages, and
  Programming, {ICALP} 2022, July 4-8, 2022, Paris, France}, volume 229 of {\em
  LIPIcs}, pages 126:1--126:13. Schloss Dagstuhl - Leibniz-Zentrum f{\"{u}}r
  Informatik, 2022.

\bibitem[\protect\citeauthoryear{Kisa \bgroup \em et al.\egroup
  }{2014}]{DBLP:conf/kr/KisaBCD14}
Doga Kisa, Guy~Van den Broeck, Arthur Choi, and Adnan Darwiche.
\newblock Probabilistic sentential decision diagrams.
\newblock In Chitta Baral, Giuseppe~De Giacomo, and Thomas Eiter, editors, {\em
  Principles of Knowledge Representation and Reasoning: Proceedings of the
  Fourteenth International Conference, {KR} 2014, Vienna, Austria, July 20-24,
  2014}. {AAAI} Press, 2014.

\bibitem[\protect\citeauthoryear{Knop}{2017}]{knop2017ips}
Alexander Knop.
\newblock Ips-like proof systems based on binary decision diagrams.
\newblock {\em Electron. Colloquium Comput. Complex.}, {TR17-179}, 2017.

\bibitem[\protect\citeauthoryear{Kushilevitz and
  Nisan}{1997}]{KushilevitzNisan}
Eyal Kushilevitz and Noam Nisan.
\newblock {\em Communication Complexity}.
\newblock Cambridge University Press, 1997.

\bibitem[\protect\citeauthoryear{Lam and
  Ruzzo}{1992}]{DBLP:journals/jcss/LamR92}
Tak~Wah Lam and Walter~L. Ruzzo.
\newblock Results on communication complexity classes.
\newblock {\em J. Comput. Syst. Sci.}, 44(2):324--342, 1992.

\bibitem[\protect\citeauthoryear{Pipatsrisawat and
  Darwiche}{2008}]{DBLP:conf/aaai/PipatsrisawatD08}
Knot Pipatsrisawat and Adnan Darwiche.
\newblock New compilation languages based on structured decomposability.
\newblock In Dieter Fox and Carla~P. Gomes, editors, {\em Proceedings of the
  Twenty-Third {AAAI} Conference on Artificial Intelligence, {AAAI} 2008,
  Chicago, Illinois, USA, July 13-17, 2008}, pages 517--522. {AAAI} Press,
  2008.

\bibitem[\protect\citeauthoryear{Pipatsrisawat and
  Darwiche}{2010}]{DBLP:conf/aaai/PipatsrisawatD10}
Thammanit Pipatsrisawat and Adnan Darwiche.
\newblock A lower bound on the size of decomposable negation normal form.
\newblock In Maria Fox and David Poole, editors, {\em Proceedings of the
  Twenty-Fourth {AAAI} Conference on Artificial Intelligence, {AAAI} 2010,
  Atlanta, Georgia, USA, July 11-15, 2010}, pages 345--350. {AAAI} Press, 2010.

\bibitem[\protect\citeauthoryear{Poon and
  Domingos}{2011}]{DBLP:conf/iccvw/PoonD11}
Hoifung Poon and Pedro~M. Domingos.
\newblock Sum-product networks: {A} new deep architecture.
\newblock In {\em {IEEE} International Conference on Computer Vision Workshops,
  {ICCV} 2011 Workshops, Barcelona, Spain, November 6-13, 2011}, pages
  689--690. {IEEE} Computer Society, 2011.

\bibitem[\protect\citeauthoryear{Segerlind}{2008}]{DBLP:conf/coco/Segerlind08}
Nathan Segerlind.
\newblock On the relative efficiency of resolution-like proofs and ordered
  binary decision diagram proofs.
\newblock In {\em Proceedings of the 23rd Annual {IEEE} Conference on
  Computational Complexity, {CCC} 2008, 23-26 June 2008, College Park,
  Maryland, {USA}}, pages 100--111. {IEEE} Computer Society, 2008.

\bibitem[\protect\citeauthoryear{Shen \bgroup \em et al.\egroup
  }{2016}]{DBLP:conf/nips/ShenCD16}
Yujia Shen, Arthur Choi, and Adnan Darwiche.
\newblock Tractable operations for arithmetic circuits of probabilistic models.
\newblock In Daniel~D. Lee, Masashi Sugiyama, Ulrike von Luxburg, Isabelle
  Guyon, and Roman Garnett, editors, {\em Advances in Neural Information
  Processing Systems 29: Annual Conference on Neural Information Processing
  Systems 2016, December 5-10, 2016, Barcelona, Spain}, pages 3936--3944, 2016.

\bibitem[\protect\citeauthoryear{Somenzi}{2009}]{somenzi2009cudd}
Fabio Somenzi.
\newblock Cudd: Cu decision diagram package release 2.4. 2.
\newblock {\em University of Colorado at Boulder}, 2009.

\bibitem[\protect\citeauthoryear{Van~{den Broeck} and
  Darwiche}{2015}]{DBLP:conf/aaai/BroeckD15}
Guy Van~{den Broeck} and Adnan Darwiche.
\newblock On the role of canonicity in knowledge compilation.
\newblock In Blai Bonet and Sven Koenig, editors, {\em Proceedings of the
  Twenty-Ninth {AAAI} Conference on Artificial Intelligence, January 25-30,
  2015, Austin, Texas, {USA}}, pages 1641--1648. {AAAI} Press, 2015.

\bibitem[\protect\citeauthoryear{Wegman and Carter}{1981}]{wegman1981new}
Mark~N Wegman and J~Lawrence Carter.
\newblock New hash functions and their use in authentication and set equality.
\newblock {\em Journal of computer and system sciences}, 22(3):265--279, 1981.

\end{thebibliography}

\end{document}